\documentclass[letterpaper, 10 pt, conference]{ieeeconf}  % Comment this line out
\IEEEoverridecommandlockouts                              % This command is only
\usepackage[utf8]{inputenc}
\usepackage[T1]{fontenc}

\usepackage{cite}
\usepackage[pdftex]{graphicx}
\usepackage{amsmath}
\usepackage{amssymb}
\usepackage{array}
\usepackage{wasysym}
\usepackage{bm}
\usepackage{siunitx}
\usepackage{url}
\usepackage{color}
\usepackage{epstopdf}
\usepackage[ruled]{algorithm2e} 
\usepackage{arydshln}
\usepackage{mathtools}

\usepackage{stfloats}
\usepackage{relsize}
\usepackage{cancel}
\usepackage{epstopdf}
\usepackage{environ}
\usepackage{amsthm}
\usepackage{bbm}
\usepackage{cancel}
\usepackage{balance}
\usepackage{dsfont}

\NewEnviron{ruledtable}{%
\par\addvspace{2mm}\hrule height 0.03cm 
\begin{table}[!h]\BODY\end{table}
\hrule height 0.03cm \addvspace{2mm}
}

\DeclareMathOperator*{\argmin}{arg\,min}

\theoremstyle{plain}

\newtheorem{definition}{Definition}
\newtheorem{proposition}{Proposition}
\newtheorem{remark}{Remark}

\newtheorem{assumption}{Assumption}
\newtheorem{lemma}{Lemma}

\usepackage{CJKutf8}

\renewcommand{\baselinestretch}{0.98}

\begin{document}
% \title{\LARGE \bf{Hierarchical Relaxation for Safety-critical Control \\ with Contradictory Safety Conditions of Quadruped Robots}}
% \title{\LARGE \bf{A Data-driven Method for Designing Control Barrier Functions \\from State Constraint Functions for Safety-critical Control}}
\title{\LARGE \bf{A Data-driven Method for Safety-critical Control: \\
Designing Control Barrier Functions from State Constraints}}
\author{Jaemin Lee, Jeeseop Kim, and Aaron D. Ames
\thanks{This work is supported by Dow under the project  \#227027AT.}
\thanks{The authors are with the Department of Mechanical and Civil Engineering, California Institute of Technology (Caltech), Pasadena, CA, USA {\tt\small \{jaemin87, jeeseop, ames\}@caltech.edu}} 
}

\maketitle

\begin{abstract}
This paper addresses the challenge of integrating explicit hard constraints into the control barrier function (CBF) framework for ensuring safety in autonomous systems, including robots. We propose a novel data-driven method to derive CBFs from these hard constraints in practical scenarios. Our approach assumes that the forward invariant safe set is either a subset or equal to the constrained set. The process consists of two main steps. First, we randomly sample states within the constraint boundaries and identify inputs meeting the time derivative criteria of the hard constraint; this iterative process converges using the Jaccard index. Next, we formulate CBFs that enclose the safe set using the sampled boundaries. This enables the creation of a control-invariant safe set, approaching the maximum attainable level of control invariance. This approach, therefore, addresses the complexities posed by complex autonomous systems with constrained control input spaces, culminating in a control-invariant safe set that closely approximates the maximal control invariant set.
\end{abstract}

% \begin{IEEEkeywords}
% Safety-Critical Control, Control Barrier Function, Contradictory Safe Sets
% \end{IEEEkeywords}

\section{Introduction}
\label{section1}
The safety of dynamical systems is of critical importance, especially in the context of autonomous systems, such as robotics, which are extensively deployed in complex and dynamic environments. Control Barrier Functions (CBFs) \cite{ames2014control, ames2016control} have emerged as a powerful tool to ensure the safety of dynamical and control systems via safety filters---optimization-based controllers that minimally modify a nominal input to produce a safe input. These techniques have been robustified to handle disturbances either by relaxing the safe set rendered invariant \cite{xu2015robustness, alan2023control, kang2023verification}, or by incorporating disturbance observers \cite{dacs2022robust, agrawal2022safe, wang2023disturbance}. Safety-critical controllers have been developed to ensure both the safety and stability of systems by amalgamating CBFs with Control Lyapunov Functions (CLFs) \cite{jankovic2018robust, taylor2020adaptive}. Furthermore, CBF-based frameworks provide an effective means to address various uncertainties in dynamical systems, including sector-bounded uncertainties \cite{buch2021robust}, stochastic uncertainties \cite{singletary2022safe}, and parametric model uncertainties \cite{taylor2020adaptive}. However, the effectiveness of these approaches critically hinges on accurate modeling of the dynamical systems, uncertainties, and disturbances.

As data-driven techniques have advanced, learning-based approaches have been introduced to enhance system safety. Supervised learning has been leveraged to ensure system safety in the presence of model uncertainty when valid CBFs are available \cite{taylor2020learning}. Similarly, CBFs have been estimated from sensor data using supervised learning techniques (support vector machine) \cite{srinivasan2020synthesis}. Imitation learning has been employed to train feedback controllers based on neural networks, which incorporate high-order CBFs \cite{yaghoubi2020training}. Deep neural networks have been used to learn CBFs from expert-demonstrated trajectories \cite{robey2020learning}. Reinforcement learning algorithms have been combined with optimization techniques or model information, including CBFs, to revise policies and prevent safety constraint violations \cite{marvi2021safe, ma2021model}. Some studies have also explored reinforcement learning for learning model uncertainty embedded in CBF and CLF constraints \cite{choi2020reinforcement}. Most of these learning-based approaches have been primarily focused on addressing issues stemming from uncertainties in models or measurements.

\begin{figure}[t] 
\centering
\includegraphics[width=\linewidth]{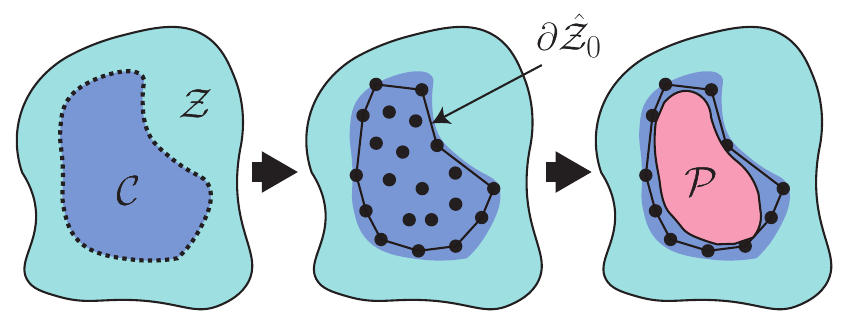}
% \vspace{-7mm}
\caption{\textbf{Overview of the proposed method}: $\mathcal{Z}$ denotes a constrained state set defined by hard constraint functions, while $\mathcal{C}$ is a forward-invariant safe set, which is unknown. $\partial \hat{\mathcal{Z}}_0$ represents the boundary of a set of data satisfying conditions that guarantee safety assurance. $\mathcal{P}$ refers to an approximated safe set associated with the CBF formulated by available data.}
\label{Fig1}
\vspace{-0.5cm}
\end{figure}

The goal of this work is to combine the advantages of data-driven approaches and the formal safety guarantees enjoyed by CBFs.  A key component of implementing CBFs in practice is the synthesis of the (forward invariant) safe set.  To enable the broader use of CBFs in practice, this paper designs CBFs that closely approximate safe sets encoding state constraints via a data-driven optimization-based approach.

% They key contribution of this paper is to utilize data-driven methods to design CBFs. When dealing with state constraints with a system model and a restricted input space, obtaining a safe set, which is forward invariant, and designing corresponding CBFs are of significant importance for safety-critical control. In many application studies, numerous CBF formulations have relied on heuristic approaches, and in some cases, these formulations have been applied without demonstrating the forward invariance of the associated safe sets. In this study, we introduce a methodology for designing CBFs that closely match the shape of the state constraints while ensuring the safe set is forward invariant, achieved through a data-driven approach employing optimization techniques.

% \textcolor{red}{Add a few sentences given the key contributions of the paper before they read more about all the related work.   Just a quick overview that you will expand upon the contributions section.}

\subsection{Related Work}

Recently, CBF-based frameworks for safety-critical control have found applications in a diverse range of autonomous systems. For example, the most prominent application of CBFs is in autonomous vehicles \cite{lyu2022adaptive, seo2022safety, choi2020reinforcement}, particularly due to the simplicity of the system model. CBFs have also been effectively applied to control simpler systems such as Segway platforms and drones, enabling obstacle avoidance \cite{taylor2020learning, molnar2021model, singletary2021comparative}. For more complex systems like legged robots, which have intricate models, defining CBFs can be challenging. Hence, reduced-order models combined with CBFs have been used to safely control legged robots such as quadruped and bipedal robots \cite{lee2023hierarchical,cosner2022safety,nguyen20163d,agrawal2017discrete}. In cases where the full-body model is considered, the CBF-based approach has been integrated with whole-body control to ensure safety in controlling complex robotic systems, such as humanoid robots \cite{khazoom2022humanoid}. While these studies demonstrate the effectiveness of CBFs, they assume the availability of valid CBFs or use CBFs defined based on physical information to render the forward-invariant safe set \cite{lee2023hierarchical,kim2023safety}. However, when considering a complex model such as a full-body model, the forward invariance of safe sets is assured using CBFs for joint limits, although CBFs for self-collisions may be used without explicit verification of forward invariance \cite{khazoom2022humanoid}. For this reason, some locomotion controller directly employed sampling-based reachability analysis to resolve the model discrepancy instead of using CBF-based approaches \cite{lee2020reachability}.

The applications of safety-critical control based on CBFs are becoming increasingly diverse and complex \cite{10266799}. While some well-established examples exist for simple models \cite{ames2016control, chen2021backup}, many application studies employ hard inequality constraints as heuristic formulations of CBFs without demonstrating that the safe set remains forward invariant, with respect to the given system model and constrained input space \cite{lee2023hierarchical, kim2023safety, nishimoto2022collision, ferraguti2022safety}. Since the forward invariance of the safe set is contingent on the system model and constrained input space, it is imperative to verify whether the set remains forward invariant. Additionally, formulating CBFs while considering hard inequality constraints, system models, and constrained input spaces is essential to ensure the safety of systems controlled by CBF-based frameworks. Hamilton–Jacobi–Bellman (HJB) methods can factor in all of the above constraints, and calculate the maximal (control) invariant set \cite{bansal2017hamilton} and an associated CBF \cite{choi2021robust}; yet these methods do not scale well with the dimension of the system.  While the connection between data-driven methods and HJB methods have been explored \cite{fisac2018general,herbert2021scalable}.

\subsection{Contribution}
This paper introduces a novel and practically realizable approach to formulating CBFs using data-driven methods, particularly when the safe set is unknown, and the constrained state space is the only information available. In contrast to most prior research, which assumed a pre-defined safe set and uncertain system models when utilizing data-driven techniques, our study addresses the challenge of deriving a safe set and the corresponding CBF when dealing only with state-dependent hard constraints. To this end, we present a systematic process for generating meaningful data with well-informed guidance to approximate a forward invariant set. Subsequently, CBFs are approximated to encapsulate a safe set, which is a subset of the forward invariant set obtained through our data-driven approach, as illustrated in Fig.\ref{Fig1}. 

The contribution of this work is threefold. Firstly, we introduce a guideline for data generation utilizing the Jaccard index. This is a significant departure from exhaustive data generation and collection, which often suffer from \textit{the curse of dimensionality}. Secondly, our approach allows for the approximation of CBFs without the need for finite-horizon learning through high-fidelity simulations. This method relies on sets approximated through data and an optimization process aimed at maximizing the set influenced by the approximated CBFs, eliminating the need to consider long-duration trajectories. Lastly, our method enables the composition of multiple sets associated with multiple CBFs accommodating hard constraints. This enables us to expand the safe set associated with a given dynamic system and constrained input space.

Our paper is organized as follows. Section \ref{section2} provides an overview of hard constraints, independent of system models and input spaces, as well as CBFs, and then we define our problem of this study. In Section \ref{section3}, we describe the proposed methods along with relevant assumptions and generalize our approach. Lastly, Section \ref{section4} provides simulation results to show the effectiveness of our methods. 

\section{Preliminaries}
\label{section2}
In this section, we lay the foundation by introducing key concepts and definitions, including hard constraint functions (HCFs) and control barrier functions (CBFs). Subsequently, we precisely articulate our problem within the context of these established concepts.

\subsection{Hard Constraints and Control Barrier Functions}
In the context of safety-critical control issues like collision avoidance, the establishment of multiple explicit hard constraints is common practice. These constraints are based on kinematic data and are primarily designed to ensure safety. Typically, these hard constraints are formulated in relation to the system's state, as illustrated below:
\begin{definition} \label{def1}
Considering a state space $\mathcal{X} \subseteq \mathbb{R}^{n}$, we define safety-critical hard constraint functions (HCFs) designed to rigorously guarantee kinematic safety. These functions operate solely on the system's state and disregard system dynamics, taking the form of $z(\bm{x}) \geq 0$, where $z$ is continuously differentiable. We establish a constrained state set denoted as $\mathcal{Z} \coloneqq {\bm{x}\in \mathcal{X} : z(\bm{x}) \geq 0 }$ to encapsulate states adhering to this constraint.
\end{definition}
\noindent
However, a critical challenge arises in that we cannot ensure the forward invariance of the set $\mathcal{Z}$. The forward invariance of sets for controlled systems is defined as follows:
\begin{definition}
    A set $\mathcal{C}$ is control (forward) invariant for a dynamical system if for any initial state $\bm{x}_0 \in \mathcal{C}$, there exist a piecewise-continuous input signal $\bm{u}(.) \in \mathcal{PU}(\mathbb{R}_{\geq0}, \mathcal{U})$ fulfilling $\bm{x}(t) \in \mathcal{C}$ for all time $t \in \mathbb{R}_{\geq0}$. Here, $\mathcal{U} \subseteq \mathbb{R}^{m}$ is a constrained input space.    
\end{definition}
To address this challenge in the context of safety-critical control problems, we frequently turn to the utilization of control barrier functions (CBFs). CBFs are instrumental in ensuring safety by incorporating both the system's dynamics and the constrained input space. This CBF-based theory begins by defining the $0$-superlevel set as $\mathcal{C}\coloneqq \{ \bm{x} \in \mathcal{X}: h(\bm{x}) \geq 0 \}$, where the function $h$ is continuously differentiable. To elaborate, let's consider a state space $\mathcal{X}$ and an input space $\mathcal{U} \subseteq \mathbb{R}^{m}$. A nonlinear system can be described as follows:
\begin{equation}
\dot{\bm{x}} = f(\bm{x}) + g(\bm{x}) \bm{u}
\end{equation}
where $\bm{x} \in \mathcal{X}$, $\bm{u} \in \mathcal{U}$, $f: \mathcal{X} \to \mathbb{R}^{n}$, and $g: \mathcal{X} \to \mathbb{R}^{n\times m}$. In addition, both $f$ and $g$ are locally Lipschitz continuous. In the study presented in \cite{ames2014control}, CBFs are defined by considering the system model and safe set as follows:
\begin{definition} \label{def2}
    Given a safe set $\mathcal{C}$, we define a function $h$ as a control barrier function (CBF) if there exists a function $\alpha\in \mathcal{K}^{e}_{\infty}$\footnote{Here, $\mathcal{K}^{e}_{\infty}$ represents the extended class of $\mathcal{K}_{\infty}$ functions. Specifically, if a continuous function $\alpha$ exhibits strict monotonicity within the open interval $(-\infty, \infty)$ while satisfying $\alpha(0) = 0$, $\alpha(-\infty) = -\infty$, and $\alpha(\infty) = \infty$, it qualifies as an extended class $\mathcal{K}{\infty}$ function.} such that for all $\bm{x} \in \mathcal{C}$, the following condition holds:
    \begin{equation*}
    \sup_{\bm{u}\in \mathcal{U}} \dot{h}(\bm{x}, \bm{u}) = \sup_{\bm{u} \in \mathcal{U}} \left[\mathcal{L}_f h(\bm{x}) + \mathcal{L}_{g} h(\bm{x}) \bm{u}) \right] \geq - \alpha(h(\bm{x}))
    \end{equation*}
    where $\mathcal{L}_{f} h(\bm{x})$ represents the Lie derivative of $h$ with respect to $\bm{x}$ and is given by $\frac{\partial h}{\partial \bm{x}}(\bm{x}) f(\bm{x})$, while $\mathcal{L}_g h(\bm{x})$ is the Lie derivative with respect to $\bm{x}$ and is expressed as $\frac{\partial h}{\partial \bm{x}}(\bm{x}) g(\bm{x})$. It is noted that the safe set $\mathcal{C}$ exhibits forward invariance in this context.
\end{definition}
\noindent
As previously highlighted, the determination of whether the defined HCFs also qualify as CBFs remains uncertain. To elaborate, if the set $\mathcal{Z}$ fails to exhibit forward invariance, the utilization of $z$ in Definition \ref{def1} as a CBF, akin to the role of $h$ in Definition \ref{def2}, becomes problematic. Strictly speaking, it is important to note that the safe set $\mathcal{C}$ is a subset of or equal to $\mathcal{Z}$,  $\mathcal{C} \subseteq \mathcal{Z}$. If the set $\mathcal{C}$ represents the largest forward invariant set, we refer to it as \textit{the maximum control invariant safe set} for satisfying the hard constraint $z(\bm{x}) \geq 0$.

\subsection{Problem Definition}
Given a HCF $z$, along with a state set $\mathcal{X}$ and an input set $\mathcal{U}$, let us consider a scenario where the forward invariance of $\mathcal{Z}$ cannot be guaranteed. In this context, we aim to design a CBF of the same order as the HCF, along with its associated safe set. This CBF is defined as follows:
\begin{equation}
    h(\bm{x}) = z(\mathbf{D}\bm{x}+\bm{c})+\varepsilon
\end{equation}
where $\mathbf{D} = \textrm{diag}(d_1, \cdots, d_n)$, $\bm{c}\in \mathbb{R}^{n}$, and $\varepsilon \in \mathbb{R}$. The safe set is defined as: $\mathcal{C} = \{\bm{x}\in \mathcal{Z}: h(\bm{x}) \geq 0\}$. Given the various potential candidates for CBFs, it becomes imperative to select one that effectively covers a broader state space while maintaining control over the system with a control input drawn from the bounded input set $\mathcal{U}$. To formalize this problem, we can express it as follows:
\begin{equation}
    \begin{split}
        \textrm{find} & \quad \mathbf{D}, \: \bm{c}, \: \varepsilon\\
        \textrm{s.t.} & \quad \mathcal{C} \textrm{ is forward invariant},\\
        & \quad d_{1}, \cdots, d_{n} \in \mathbb{R}, \: \bm{c} \in \mathbb{R}^{n}, \: \varepsilon \in \mathbb{R}.
    \end{split}
\end{equation}
This paper offers a solution to the aforementioned problem through the utilization of data-driven techniques and convex approximations. Moreover, it facilitates the acquisition of a safe set that is forward invariant by combining multiple barrier functions through logical operators while resolving our problem. 

\section{The Proposed Data-driven Methods}
\label{section3}
In this section, we introduce data-driven approaches to determine the constant variables $\mathbf{D}$, $\bm{c}$, and $\varepsilon$ for a single hard constraint. Subsequently, we extend this method to address multiple CBFs. Notably, Definition \ref{def1} establishes a correlation between the lower bound of the time derivative of the CBF and the forward invariance of the set. From this definition, it becomes evident that a crucial characteristic is the requirement for the lower bound of the time derivative of the function to be $0$ at the boundary of the forward invariant set, denoted as $\partial \mathcal{C} \coloneqq \{ \bm{x} \in \mathcal{X}: h(\bm{x}) = 0 \}$:
\begin{equation} \label{CBF_ineq_bound}
    \sup_{\bm{u}\in \mathcal{U}} \dot{h}(\bm{x}, \bm{u}) \geq 0 \quad \forall \bm{x} \in \partial \mathcal{C},
\end{equation}
given that for any function $\alpha$, we have $\alpha(h(\bm{x})) = 0$ for all $\bm{x}\in \partial \mathcal{C}$, the inequality above implies the following statement:
\begin{equation} \label{CBF_eq_bound}
    \exists \bm{u} \in \mathcal{U}: \dot{h}(\bm{x}, \bm{u})=0 \quad \forall \bm{x} \in \partial \mathcal{C}.
\end{equation}
In our context, we start with the assumption that we have a candidate CBF (HCF) denoted as $z(\bm{x})$ as the only available information. Therefore, it becomes necessary to rephrase the conditions expressed in the above inequality using the function $z$.

\subsection{Uniform scaling and Offset}
To simplify our problem, we introduce the following assumption to restate the aforementioned conditions in terms of $z$, particularly in cases where the safe set $\mathcal{C}$ is defined by the function $h$:
\begin{assumption} \label{assumption1}
The CBF $h$ is derived from the HCF $z$ with the introduction of a constant offset $\varepsilon$ and a scaling parameter $d$. Under this framework, we make an assumption regarding the partial derivatives of these functions, which can be expressed as follows:
\begin{equation}
    \frac{\partial h}{\partial \bm{x}} (\bm{x}) =  \frac{\partial}{\partial \bm{x}} z(\mathbf{D}\bm{x} + \bm{c}) = d \frac{\partial z}{\partial \bm{x}}(\bm{x}).
\end{equation}
Here, it's important to note that $d$ is strictly greater than zero. For instance, if we consider the HCF to be a linear function, such as $z(\bm{x}) = \mathbf{A}\bm{x} + b$, and assume that all diagonal components of $\mathbf{D}$ are set to $d$, the equation above holds as follows:
\begin{equation}
    \frac{\partial }{\partial \bm{x}} \left( \mathbf{A} (\mathbf{D}\bm{x} + \bm{c}) + b \right) = \mathbf{A}\mathbf{D} = d \mathbf{A} = d \frac{\partial z}{\partial \bm{x}} (\bm{x}).
\end{equation}
\end{assumption}
\noindent
Based on Assumption \ref{assumption1}, we can express the conditions pertaining to boundary states in the safe set $\partial \mathcal{C}$, defined as ${\bm{x} \in \mathcal{X}: h(\bm{x}) = 0 }$, in the following manner:
\begin{equation}
    \sup_{\bm{u} \in \mathcal{U}} \dot{z}(\bm{x}, \bm{u}) = \sup_{\bm{u} \in \mathcal{U}} d^{-1} \dot{h}(\bm{x}, \bm{u}) \geq 0, \quad \forall \bm{x}\in \partial \mathcal{C}.
\end{equation}
The above inequality holds true since $\sup_{\bm{u}\in \mathcal{U}} \dot{h}(\bm{x}, \bm{u}) \geq 0$ for all states $\bm{x} \in \partial \mathcal{C}$ and $d>0$. Consequently, the states can be elements of $\partial \mathcal{C}$ as long as they satisfy $\dot{z}(\bm{x}, \bm{u}) =0$. With this premise, we introduce a subset of $\mathcal{Z}$ defined as follows:
\begin{equation}
    \mathcal{Z}_{0} = \{ \bm{x} \in \mathcal{Z} : \exists \bm{u} \in \mathcal{U} \quad \textrm{s.t.} \quad \dot{z}(\bm{x}, \bm{u}) = 0 \}.
\end{equation}
While obtaining $\mathcal{Z}_0$ can be straightforward for certain types of systems and input spaces, the complexity and nonlinearity of system dynamics render the analytical process of acquiring $\mathcal{Z}_0$ challenging and often unfeasible. Hence, in this section, we use a data-driven approach to approximate $\mathcal{Z}_0$.

The data-driven method involves two sequential steps. First, we randomly sample state points $\bm{x} \sim U(\underline{\bm{a}}, \overline{\bm{a}})$, where $\underline{\bm{a}} \in \mathbb{R}^{n}$ and $\overline{\bm{a}} \in \mathbb{R}^{n}$ are the upper and lower bounds of the state space, respectively. These sampled states constitute a set $\mathcal{S}$. Subsequently, we assess whether these sampled states satisfy the hard constraint $z(\bm{x}) \geq 0$ and retain the ones that do:
\begin{equation}
    \mathcal{Z}_{s} = \{\bm{x} \in \mathcal{S}: z(\bm{x}) \geq 0 \}
\end{equation}
where $\mathcal{Z}_{s} = \mathcal{S} \bigcap \mathcal{Z}$. In the second step, we identify state samples within $\mathcal{Z}_s$ that possess at least one control input leading to $\dot{z}(\bm{x}, \bm{u}) = 0$ within the constrained input set $\mathcal{U}$:
\begin{equation} \label{opt_qp_zero}
    \hat{\mathcal{Z}}_0 = \{ \bm{x} \in \mathcal{Z}_{s}: \exists \bm{u} \in \mathcal{U}, \quad \textrm{s.t.} \quad \dot{z}(\bm{x}, \bm{u})=0  \}
\end{equation}
where $\hat{\mathcal{Z}}_{0} \subset \mathcal{Z}_{0}$. To determine the existence of a feasible control input for a given state sample $\bm{x} \in \mathcal{Z}_{s}$, we can generally solve the following optimization problem:
\begin{equation}
    \begin{split}
        \min \quad \| \mathcal{L}_{f}z(\bm{x}) + \mathcal{L}_{g}z(\bm{x}) \bm{u}\|^{2}, \quad \textrm{s.t.} \quad \bm{u} \in \mathcal{U} .
    \end{split}
\end{equation}
When the optimal objective of the above problem is $0$, the state $\bm{x}$ belongs to $\hat{\mathcal{Z}}_{0}$. It is important to note that $\hat{\mathcal{Z}}_{0}$ constitutes a discrete set, and the effectiveness of its approximation depends on the sampling strategy employed.

A reasonable number of samples is required to obtain an accurate approximation of the set $\hat{\mathcal{Z}}_0$. To assess the convergence of the sampling process, we employ the Jaccard index, defined as follows: 
\begin{definition}
Given two discrete sets $\mathcal{A}_1$ and $\mathcal{A}_2$, the Jaccard index of the sets is defined as the size of their intersection divided by the size of their union:
    \begin{equation}
        \mathcal{J}(\mathcal{A}_1, \mathcal{A}_2) = \frac{|\mathcal{A}_1 \bigcap \mathcal{A}_2|}{|\mathcal{A}_1 \bigcup \mathcal{A}_2 |} = \frac{\textrm{card}(\mathcal{A}_1 \bigcap \mathcal{A}_2)}{\textrm{card}(\mathcal{A}_1 \bigcup \mathcal{A}_2)}
    \end{equation}
    where $|\mathcal{A}|$ denotes the size of set $\mathcal{A}$, and we utilize the term \textit{cardinality} denoted by $\textrm{card}(\mathcal{A})$ to represent it.
\end{definition}
\noindent
Considering that $\hat{\mathcal{Z}}_0 \subseteq \mathcal{S}$, we compute the Jaccard index using the sets $\hat{\mathcal{Z}}_{0}$ and $\mathcal{S}$ as follows: $\mathcal{J}(\hat{\mathcal{Z}}_0, \mathcal{S}) = \frac{\textrm{card}(\hat{\mathcal{Z}}_{0})}{\textrm{card}(\mathcal{S})}$, where $0 \leq \mathcal{J}(\hat{\mathcal{Z}}_{0}, \mathcal{S}) \leq 1$. To explicitly consider the number of samples denoted as $N_s$, we employ $\mathcal{J}_{[N_s]}(\hat{\mathcal{Z}}_0, \mathcal{S})$ as the Jaccard index, where $\textrm{card}(\mathcal{S}) = N_s$.
\begin{proposition} \label{prop1}
If we draw an infinite number of state data to define $\hat{\mathcal{Z}}_{0}$, $\mathcal{J}_{[N_s]}(\hat{\mathcal{Z}}_0, \mathcal{S})$ converges to a constant number between $0$ and $1$, signifying that the set $\hat{\mathcal{Z}}_{0}$ approaches $\mathcal{Z}_0$.
\end{proposition}
\begin{proof}
To demonstrate the convergence of $\mathcal{J}_{[N_s]}(\hat{\mathcal{Z}}_0, \mathcal{S})$ as $N_s$ approaches infinity, we examine the magnitude of the Jaccard index gradient:
    \begin{equation*}
    \begin{split}
        \Delta \mathcal{J}_{[N_s]} (\hat{\mathcal{Z}}_0, \mathcal{S})  &= |\mathcal{J}_{[N_s+\Delta N_s]}(\hat{\mathcal{Z}}_{0}, \mathcal{S}) - \mathcal{J}_{[N_s]}(\hat{\mathcal{Z}}_{0}, \mathcal{S})| \\
        &\leq \left|\frac{N_{z} + \Delta N_s}{N_s + \Delta N_s} - \frac{N_z}{N_s} \right|
        \end{split}
    \end{equation*}
    where $ \Delta \mathcal{J}_{[N_s]} (\hat{\mathcal{Z}}_0, \mathcal{S})\geq 0$ and $N_s \gg \Delta N_s > 0$. In addition, $N_z$ is the cardinarity of $\hat{\mathcal{Z}}_0$ when computing $\mathcal{J}_{[N_s]}(\hat{\mathcal{Z}}_0, \mathcal{S})$. Then, we express the upper bound of $\Delta \mathcal{J}_{[N_s]}(\hat{\mathcal{Z}}_0, \mathcal{S})$ as $N_s$ goes to $\infty$: 
    \begin{equation*}
        \begin{split}
            \lim_{N_s \to \infty} \Delta \mathcal{J}_{[N_s]} (\hat{\mathcal{Z}}_0, \mathcal{S}) \leq&  \lim_{N_s \to \infty} \left|\frac{N_s \Delta N_s - N_z \Delta N_s}{N_s(N_s + \Delta N_s)} \right|\\
            \leq & \lim_{N_s \to \infty} \left| \frac{\Delta N_s}{N_s + \Delta N_s}\right|=0. 
        \end{split}
    \end{equation*}
Since we assume the state samples are uniformly distributed, this convergence of $\Delta \mathcal{J}_{[N_s]}(\hat{\mathcal{Z}}_0, \mathcal{S})$ indicates that $\lim{N_s\to \infty} \mathcal{J}_{[N_s]} (\hat{\mathcal{Z}}_0, \mathcal{S}) = \mathcal{J}(\mathcal{Z}_0, \mathcal{Z})$, which implies that $\hat{\mathcal{Z}}_0 = \mathcal{Z}_0$ and $\mathcal{S} = \mathcal{Z}$.
\end{proof}
\noindent
As mentioned in Proposition \ref{prop1}, generating a large number of samples is ideal for obtaining $\hat{\mathcal{Z}}_0$, which closely approximates $\mathcal{Z}_0$. However, the need for computational efficiency requires us to establish two conditions for this sampling procedure:
\begin{equation}
    N_s \geq N_{s,\min}, \quad \Delta \mathcal{J}_{[N_s]}(\hat{\mathcal{Z}}_0,\mathcal{S})\leq \delta
\end{equation}
where $N_{s,\min}$ and $\delta$ are the predefined minimum number of samples and threshold for the magnitude of Jaccard index gradient. 

Once we have obtained $\hat{\mathcal{Z}}_0$, we collect a boundary set of $\hat{\mathcal{Z}}_0$, defined as follows:
\begin{equation}
    \begin{split}
        \partial \hat{\mathcal{Z}}_0 \coloneqq \{ \bm{x} \in \hat{\mathcal{Z}}_0:& \exists (\bm{y}_1, \bm{y}_2)\textrm{ where }\bm{y}_1 \in \hat{\mathcal{Z}}_{0}, \bm{y}_2 \notin \hat{\mathcal{Z}}_0  
 \\
 &\textrm{ within the }\epsilon\textrm{-neighborhood of }\bm{x}\}
    \end{split}
\end{equation}
where $\epsilon>0$. It holds that $\partial \hat{\mathcal{Z}}_{0} \subset \mathcal{C}$, although we cannot guarantee that $\partial \hat{\mathcal{Z}}_0 \subseteq \partial \mathcal{C}$.
\begin{assumption}
$\partial \hat{\mathcal{Z}}_0$ is open, so we assume that $\hat{\mathcal{Z}}_0$ is dense enough, and there exists an approximated set $\partial \widetilde{\mathcal{Z}}_0 \subset \partial \hat{\mathcal{Z}}_0$ that is closed where its elements satisfy that the optimal cost of \eqref{opt_qp_zero} is zero. We define a set $\widetilde{\mathcal{Z}}_0 \coloneqq \partial \widetilde{\mathcal{Z}}_0 \cup \textrm{int}(\widetilde{\mathcal{Z}}_0)$ where $\hat{\mathcal{Z}}_0 \subset \widetilde{\mathcal{Z}}_0$ and forward invariant.   
\end{assumption}
\noindent
Now, the unknown variables $d$, $\bm{c}$, and $\varepsilon$ are determined to maximize an approximated set $\mathcal{P} \subset \mathcal{C}$, which lies inside the boundary $\partial \hat{\mathcal{Z}}_0$ (in cases where $\mathcal{C}$ is closed).
\begin{equation} \label{opt_offset}
    \begin{split}
        \max_{d, \bm{c}, \varepsilon} & \quad  | \mathcal{P} | = \iiint_V z(d \bm{x} +  \bm{c}) + \varepsilon \; dV \\
        \textrm{s.t.} & \quad z( d \bm{x}_{k} + \bm{c}) + \varepsilon < 0, \quad \forall \bm{x}_{k} \in \partial \hat{\mathcal{Z}}_0.
    \end{split}
\end{equation}
where $V$ is the $n$-dimensional region over which we are integrating, and $dV$ represents an infinitesimal $n$-dimensional volume element. If the set $\mathcal{C}$ is not closed, we need to adjust $V$ to make the cost finite in the above optimization. Using the optimal solution to the above optimization problem $d^{\star}$, $\bm{c}^{\star}$, and $\varepsilon^{\star}$, we define a CBF, $h_{\textrm{approx}}(\bm{x}) = z(d^{\star}\bm{x} + \bm{c}^{\star}) + \varepsilon^{\star}$.
\begin{lemma} \label{lemma1}
The set defined by the approximated CBF, $\mathcal{P}_{\textrm{approx}} = \{ \bm{x}\in \mathcal{X}: h_{\textrm{approx}}(\bm{x}) \geq 0 \}$, is forward invariant. 
\end{lemma}
\begin{proof}
The dynamical system is locally Lipschitz continuous, and $\hat{\mathcal{Z}}_0$ is dense enough to satisfy that the control input $\bm{u}(\tau) \in \mathcal{U}$ exists with $\bm{x}_0 \in \widetilde{\mathcal{Z}}_0$ and $\bm{x}_{z}\in \hat{\mathcal{Z}}_0$ for all $\tau \in [t_0, t]$ such that:
\begin{equation}
\bm{x}_z = \bm{\Phi}(t, t_0) \bm{x}_0 +\int_{t_0}^{t} \bm{\Phi}(t, \tau)g(\bm{x}(\tau)) \bm{u}(\tau) d\tau
\end{equation}
where $\bm{\Phi}(t,t_0)$ denotes the state transition matrix describing the evolution of state from $t_0$ to $t$ where $t$ is finite. In this case, $\mathcal{P}_{\textrm{approx}} \subseteq \widetilde{\mathcal{Z}}_0$, $\mathcal{P}_{\textrm{approx}}$ is forward invariant, since $\widetilde{\mathcal{Z}}_0$ is forward invariant.
\end{proof}

The set defined by the approximated CBF, $\mathcal{P}_{\textrm{approx}} = \{ \bm{x}\in \mathcal{X}: h_{\textrm{approx}}(\bm{x}) \geq 0 \}$, being forward invariant, we utilize $h_{\textrm{approx}}$ as a CBF for synthesizing an optimization-based controller with the desired output value $\bm{\xi} \in \mathbb{R}^{p}$:
\begin{equation} \label{safety_filter_single}
\begin{split}
    \mathbf{k}(\bm{x},\bm{\xi}) = \argmin_{\bm{u} \in \mathcal{U}} & \quad \|\mathbf{k}^{d}(\bm{x}, \bm{\xi}) - \bm{u}  \|^{2},\\
    \textrm{s.t.}  &\quad \dot{h}_{\textrm{approx}}(\bm{x}, \bm{u}) \geq - \alpha(h_{\textrm{approx}}(\bm{x}))
\end{split}
\end{equation}
where $\mathbf{k}^d: \mathcal{X} \times \mathbb{R}^{p} \to \mathbb{R}^{m}$ represents the desired controller, which is locally Lipschitz continuous. It is important to note that the initial state must be inside $\mathcal{P}_{\textrm{approx}}$ to ensure safety. 

\subsection{Non-uniform Scaling}
We can introduce a variation of the CBF with scaling coefficients. The following proposition demonstrates that the scaling coefficients of the candidate function must satisfy specific conditions to be utilized in obtaining the CBF. We introduce another assumption for deriving the CBF with non-uniform scaling of the HCF.
\begin{assumption} \label{assumption2}
The control barrier function $h$ is defined by using non-uniform scaling parameters $d_1 \geq 0, \cdots, d_{n} \geq 0$ in $\mathbf{D}$. Then, we assume that 
    \begin{equation}
        \frac{\partial h}{\partial \bm{x}}(\bm{x}) = \frac{\partial}{\partial \bm{x}}z(\mathbf{D}\bm{x} + \bm{c}) = \frac{\partial z}{\partial \bm{x}}(\bm{x}) \mathbf{D}.
    \end{equation} 
\end{assumption}
\noindent
Now, based on the above assumption, we reformulate the condition for boundary states of the safe set $\mathcal{C}$ as presented in the following proposition.
\begin{proposition} \label{prop2}
Assume that the boundary of the sampled set, $\partial \hat{\mathcal{Z}}_0$, is obtained. Then, $\partial \hat{\mathcal{Z}}_0$ can be considered as the boundary of a safe set if there exists at least one control input $\bm{u} \in \mathcal{U}$ satisfying $\frac{\partial z}{\partial \bm{x}}(\bm{x}) \overline{\mathbf{D}} \dot{\bm{x}} \geq 0$ for all $\bm{x} \in \partial \hat{\mathcal{Z}}_{0}$, where $\overline{\mathbf{D}} = \textrm{diag}(0,d_2-d_1, \cdots, d_n - d_1)$.
\end{proposition}
\begin{proof}
    First, we express $\dot{h}(\bm{x},\bm{u})$ with $z$ as follows:
    \begin{equation} \label{prop2_ineq}
        \begin{split}
        \dot{h}(\bm{x},\bm{u}) =&  \frac{\partial z}{\partial \bm{x}}(\bm{x}) \mathbf{D}\dot{\bm{x}} = \sum_{i=1}^{n}  d_{i} \frac{\partial z}{\partial \bm{x}_{i}}(\bm{x})  \dot{\bm{x}}_{i} \geq 0. 
        \end{split}
    \end{equation}
where $\bm{x}_i$ and $\dot{\bm{x}}_{i}$ denote the $i$-th elements of $\bm{x}$ and $\dot{\bm{x}}$, respectively. Second, the states in $\partial \hat{\mathcal{Z}}_0$ guarantee the existence of control inputs holding $\frac{\partial z}{\partial \bm{x}}(\bm{x})\dot{\bm{x}} = 0$, which implies $d_1 \frac{\partial z}{\partial \bm{x}_1}(\bm{x}) \dot{\bm{x}}_1 = - d_1 (\sum_{i=2}^{n} \frac{\partial z}{\partial \bm{x}_i}(\bm{x}) \dot{\bm{x}}_{i})$. By substituting $d_1 \frac{\partial z}{\partial \bm{x}_1}(\bm{x}) \dot{\bm{x}}_1$ into \eqref{prop2_ineq}, we have the following inequality:
    \begin{equation}
    \begin{split}
        \sum_{i=1}^{n}  d_{i} \frac{\partial z}{\partial \bm{x}_{i}}(\bm{x})  \dot{\bm{x}}_{i} &= \sum_{i=2}^{n} (d_i - d_1) \frac{\partial z}{\partial \bm{x}_{i}}(\bm{x})  \dot{\bm{x}}_{i} \\
        &= \frac{\partial z}{\partial \bm{x}}(\bm{x}) \overline{\mathbf{D}} \dot{\bm{x}} \geq 0.
        \end{split}
    \end{equation}
    for all the states belonging to $\partial \hat{\mathcal{Z}}_0$.
\end{proof}

Taking into account Assumption \ref{assumption2} and Proposition \ref{prop2}, we can modify the optimization problem \eqref{opt_offset} to determine the variables for the approximated CBF as follows:
\begin{align} 
    \max_{\mathbf{D}, \bm{c}, \varepsilon} & \quad  | \mathcal{P} | = \iiint_V z(\mathbf{D} \bm{x} +  \bm{c}) + \varepsilon \; dV  \nonumber \\
    \textrm{s.t.} & \quad z( \mathbf{D} \bm{x}_{k} + \bm{c}) + \varepsilon < 0, \quad \forall \bm{x}_{k} \in \partial \hat{\mathcal{Z}}_0,  \label{opt_nonuniform} \\
    & \quad \frac{\partial z}{\partial \bm{x}}(\bm{x}_{k}) \overline{\mathbf{D}} (f(\bm{x}_{k}) + g(\bm{x}_{k}) \bm{u})\geq 0, \quad \bm{u}_{k} \in \mathcal{U}. \nonumber
\end{align}
This problem is relatively larger in size than \eqref{opt_offset}, given the same system and the same order of HCF. However, it has the potential to define a CBF capturing a larger safe set compared to \eqref{opt_offset}. Importantly, since Lemma \ref{lemma1} holds in this case for the safe set associated with the approximated CBF obtained from \eqref{opt_nonuniform}, the subsequent steps for safety-critical control remain identical to the method presented in the previous section. This involves solving a QP problem with inequality constraints to ensure safety.
 
\subsection{Multiple CBFs}
We have previously assumed that a single CBF can represent a safe set associated with an HCF. However, it is possible to improve the accuracy of this representation by employing multiple CBFs and logical operators, such as AND ($\wedge$) and OR ($\vee$). In this section, we continue to discuss the data-driven process for obtaining $\partial \hat{\mathcal{Z}}_0$. Given that the non-uniform scaling method is a more general approach, we focus on extending this method to multiple CBFs. Suppose we have $s$ CBFs representing a safe set with an HCF. We can formulate the following optimization problem:
\begin{subequations} \label{multi_opt}
    \begin{align}
        \max_{\bm{\Psi}}&\quad \left|\mathcal{P}_{\cap} \coloneqq \bigcap_{j=1}^{s} \mathcal{P}_{j}\right|   \\
        \textrm{s.t.}&\quad  z(\mathbf{D}_{j} \bm{x} + \bm{c}_j) + \varepsilon_{j} \geq 0,\quad \forall \bm{x} \in \mathcal{P}_{\cap} \subset \widetilde{\mathcal{Z}}_{0} \label{eq_multi_z}\\
        &\quad \frac{\partial z}{\partial \bm{x}}(\bm{\nu}) \overline{\mathbf{D}}_{j} (f(\bm{\nu}) + g(\bm{\nu}) \bm{u})\geq 0, \quad \forall \bm{\nu} \in \partial \mathcal{P}_{j} \label{eq_multi_dz} \\
        &\quad \bm{u} \in \mathcal{U}, \quad j \in \{ 1, \cdots, s \}.
    \end{align}
\end{subequations}
where $\bm{\Psi} = (\Psi_{1}, \cdots, \Psi_{s})$ represents a collection of tuples containing decision variables, where $\Psi_{j} = (\mathbf{D}_{j}, \bm{c}_j, \varepsilon_{j})$. The constraints \eqref{eq_multi_z} imply that the approximated safe set is defined as $\mathcal{P}_{\cap} = \{ \bm{x} \in \widetilde{\mathcal{Z}}_0 : (h_{\textrm{approx},1}(\bm{x}) \geq 0) \wedge \cdots \wedge (h_{\textrm{approx},s}(\bm{x}) \geq 0) \}$. Regarding the general form of this approach, we offer the following remarks to clarify the significance of the number of CBFs in the approximation process:
\begin{remark} \label{remark1}
If the function $z(\bm{x})$ has an order greater than 1, and there are no identical tuples in the optimal solution to \eqref{multi_opt} involving $s$ number of CBFs, the optimal cost will exceed that of the solution to \eqref{multi_opt} with $\ell$ number of CBFs, where $\ell < s$. This implies that if identical tuples do exist in the optimal solution with $s$ number of CBFs, it is possible to define a safe set with the same optimal cost using only $s-1$ number of CBFs.
\end{remark}

\begin{remark} \label{remark2}
When dealing with a linear HCF expressed as $z(\bm{x}) = \mathbf{A}\bm{x} + \bm{b}$, it becomes necessary to examine each tuple to determine whether it satisfies the condition stated in Remark \ref{remark1}. In particular, we need to verify that $h_{\textrm{approx},1}(\bm{x}) \neq \zeta h_{\textrm{approx},2}(\bm{x})$, where $\zeta > 0$. For instance, we can specify $h_{\textrm{approx},1}(\bm{x})$ as $\mathbf{A} \mathbf{D}_1 \bm{x} + \mathbf{A}\bm{c}_{1} + \varepsilon_{1}+c$ and $h_{\textrm{approx},2}(\bm{x})$ as $\mathbf{A} \mathbf{D}_2 \bm{x} + \mathbf{A}\bm{c}_{2} + \varepsilon_{2}+c$. If conditions $\mathbf{D}_1 = \zeta \mathbf{D}_2$ and $\mathbf{A}\bm{c}_{1} + \varepsilon_{1}+c = \zeta(\mathbf{A}\bm{c}_{2} + \varepsilon_{2}+c)$ hold true, then it's important to note that one CBF among them only holds significance within the approximated safe set.
\end{remark}

Similar to the safety filter described in \eqref{safety_filter_single}, we formulate the QP problem involving the approximated CBFs. Given the desired output value $\bm{\xi}$ and the desired controller $\bm{k}^{d}$, the QP problem is formulated as follows:
    \begin{align}
        \bm{k}(\bm{x}, \bm{\xi}) = \argmin_{\bm{u}\in \mathcal{U}}& \quad \| \bm{k}^{d}(\bm{x}, \bm{\xi}) - \bm{u} \|^{2},\\
        &\quad  \dot{h}_{\textrm{approx},i}(\bm{x}, \bm{u}) \geq - \alpha_{i}(h_{\textrm{approx},i}(\bm{x})), \nonumber \\
        & \quad \forall i \in \{1, \cdots, s\} \nonumber
    \end{align}
where $\bm{x} \in \textrm{int}(\mathcal{P}_{\cap})$. It is crucial to select appropriate functions $\alpha_{i}$ to ensure the safety of the systems controlled by $\bm{k}(\bm{x}, \bm{\xi})$. 

\section{Validation}
\label{section4}
This section aims to validate the proposed data-driven approach using a simple system with a constrained input space. Our example involves a double integrator system with a state vector $\bm{x} = [ x, \dot{x}]^{\top}$, where $x$ and $\dot{x}$ represent position and velocity, respectively. The state space model of the closed-loop system can be described as:
\begin{equation}
    \dot{\bm{x}} = \underbrace{\left[\begin{array}{cc} 0 & 1 \\ 0 & 0 \end{array} \right]\bm{x}}_{f(\bm{x})} + \underbrace{\left[\begin{array}{c}  0  \\ 1 \end{array} \right]}_{g(\bm{x})} u 
\end{equation}
where $u \in [u_{\min}, u_{\max}]$ and $\dot{x} \in [v_{\min}, v_{\max}]$. The HCF, in this case, is given by $z(\bm{x}) = \gamma_1 - x \geq 0$, with a constant $\gamma_1$. However, there are two critical issues with this setup: 1) $\dot{z}$ is not a function of $u$, and 2) We cannot guarantee $\mathcal{Z}$ is forward invariant where $\mathcal{Z} \coloneqq \{\bm{x} : z(\bm{x}) \geq 0 \}$. Due to inertia, the limited control input cannot prevent the state from crossing the boundary of $\mathcal{Z}$. To address this, we introduce an additional damping term to prevent penetration near the boundary of the hard constraint:
\begin{equation}
    z(\bm{x}) = \gamma_1 - x -  \mathds{1}_{\dot{x}> 0} \gamma_2 \dot{x} 
\end{equation}
where $\gamma_2\geq0$ and $\mathds{1}_{\dot{x}>0}$ is an indicator function that takes the value of $1$ when $\dot{x} > 0$ and $0$ otherwise. We can then compute the time derivative of $z$ as follows:
\begin{equation}
    \dot{z}(\bm{x}, u)= - \dot{x} - \mathds{1}_{\dot{x}>0} \gamma_2 u.
\end{equation}
If $\dot{x}<0$, there is no control input that can make $\dot{z}(\bm{x}, u) = 0$, which implies that $\dot{z}(\bm{x}, u) >0$ for all $\dot{x}<0$. When $\dot{x}>0$, the time derivative of the candidate function becomes:
\begin{equation}
    \dot{z}(\bm{x}, u) = -\dot{x} - \gamma_{2} u.
\end{equation}
In this case, the control command that satisfies $\dot{z}(\bm{x}, u) = 0$ can be found using a simple algebraic equation: $u = - \frac{\dot{x}}{\gamma_2}$. To make the optimization problem more tractable in this example, we consider the set $\hat{\mathcal{Z}}_0$ as follows:
\begin{equation}
\begin{split}
    \hat{\mathcal{Z}}_{0} = & \{\bm{x} \in \mathcal{S}: \exists u \in \mathcal{U} \; \textrm{ s.t. } \; \dot{z}(\bm{x}, u) =0 \wedge \dot{x} \geq 0\}\\
    & \bigcap \{ \bm{x} \in \mathcal{S}: \dot{x} <0\}.
\end{split}
\end{equation}
This set includes states $\bm{x} \in \mathcal{S}$ with $\dot{x}<0$, ensuring that $\dot{z}(\bm{x}, u) >0$ regardless of the control input $u$.  

\begin{figure}[t] 
\centering
\includegraphics[width=\linewidth]{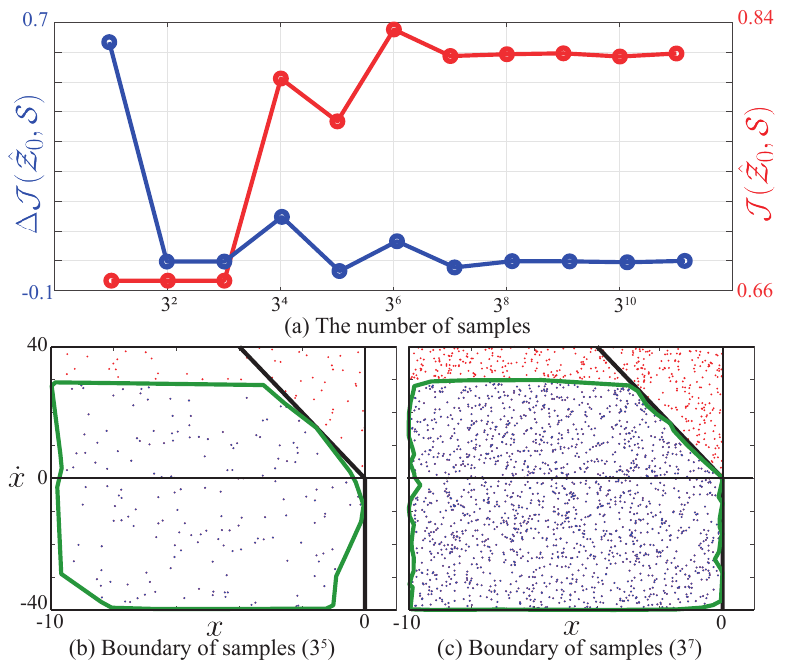}
% \vspace{-7mm}
\caption{\textbf{Data generation and boundary of data set}: (a) Jaccard index ($\mathcal{J}(\hat{\mathcal{Z}}_0, \mathcal{S})$) and variation of Jaccard index ($\Delta\mathcal{J}(\hat{\mathcal{Z}}_0, \mathcal{S})$) as a function of the number of data. (b) and (c) State space data with data sizes of $3^5$ and $3^7$, respectively. Blue and red dots represent data in $\hat{\mathcal{Z}}_0$ and $\mathcal{S} - \hat{\mathcal{Z}}_0$, While black bold lines depict the boundary defined by $z(\bm{x}) = 0$. }
\label{Fig2}
\vspace{-0.5cm}
\end{figure}

\begin{figure}[t] 
\centering
\includegraphics[width=\linewidth]{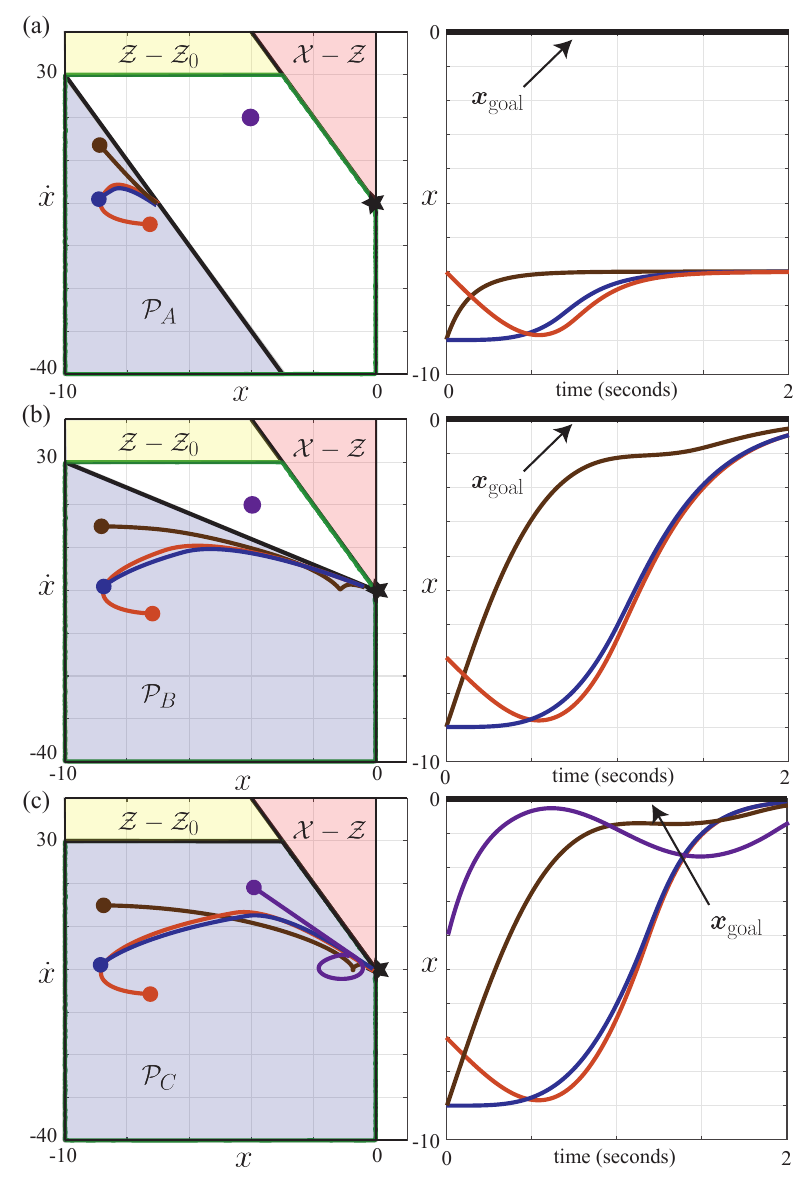}
% \vspace{-7mm}
\caption{\textbf{Safety-critical control results in the state space (Left) and position space (Right)}: (a) Uniformly scaling and offset (III-A), (b) Non-uniform scaling (III-B), (c) Multiple CBFs (III-C). Color-coded results with different initial states: Brown: $\bm{x}_{\textrm{init}}=(-9,15)$, Blue: $\bm{x}_{\textrm{init}}=(-9,0)$, Red: $\bm{x}_{\textrm{init}}=(-7,-5)$, Purple: $\bm{x}_{\textrm{init}}=(-4,20)$.}
\label{Fig3}
\vspace{-0.5cm}
\end{figure}

For the demonstrations, we have chosen specific parameter values: $\gamma_1 = 0$, $\gamma_2 = 0.1$, and the input limits $[u_{\min}, u_{\max}] = [-300,\; 300]$. Data is generated from a uniform distribution within the bounds $x \in [-10,\; 0]$ and $\dot{x} \in [-40,\; 40]$, and we collected data with $N_{s,\min} = 1000$ and $\delta = 0.001$. As depicted in Fig. \ref{Fig2}(a), the Jaccard index converges to a specific constant, and we choose a sample size of $3^{11}$ based on the defined $N_{s, \min}$ and $\delta$. In Figs. \ref{Fig2}(b) and \ref{Fig2}(c), we visualize the states and the boundary of $\hat{\mathcal{Z}}_0$ with different numbers of data points, specifically $3^{5}$ and $3^{7}$. Within these figures, blue dots indicate states within $\hat{\mathcal{Z}}_0$ that lie within the integration region, while a green line represents its boundary $\partial \hat{\mathcal{Z}}_0$. Furthermore, red dots signify states in $\mathcal{S} - \hat{\mathcal{Z}}_0$, where $\dot{z}(\bm{x}, u) < 0$ for any control input $u \in \mathcal{U}$ when $\bm{x} \in \mathcal{S} - \hat{\mathcal{Z}}_0$. As the number of data points increases, the boundary becomes simpler in this example.

With ample data, the state space can be divided into several polytopes, as shown in Fig. \ref{Fig3}. These polytopes include $\mathcal{X} - \mathcal{Z}$ (red), $\mathcal{Z} - \mathcal{Z}_0$ (yellow), $\partial \mathcal{Z}_0$ (green line), and its interior. More specifically, states within the red area fail to satisfy the state hard constraint. The states in the yellow region satisfy the constraint but are considered unsafe since the set is not forward invariant. The QP problem described in \eqref{safety_filter_single} with $z$ becomes infeasible in such cases. To clearly define a safe set, which is forward invariant with the satisfaction of the constraints, we propose three approximated CBFs using our method: Uniform scaling and offset (III-A, Fig. \ref{Fig3}(a)), Non-uniform scaling (III-B, Fig \ref{Fig3}(b)), and Multiple CBFs (III-C, Fig \ref{Fig3}(c)):
\begin{subequations}
    \begin{align}
        \textrm{III-A: }& h_{\textrm{approx}}(\bm{x}) = -x - \mathds{1}_{\dot{x}>-70} (0.1\dot{x} + 7)  \\
        \textrm{III-B: }& h_{\textrm{approx}}(\bm{x}) = -x - \mathds{1}_{\dot{x}>0} \frac{\dot{x}}{3}\\
        \textrm{III-C: }& h_{\textrm{approx},1}(\bm{x}) = -x - \mathds{1}_{\dot{x}>0} 0.1 \dot{x}  \\
        & h_{\textrm{approx},2}(\bm{x}) = - \dot{x} + 30.
    \end{align}
\end{subequations}
The safe sets defined by the approximated CBFs are represented by blue regions in Fig. \ref{Fig3}. To validate the defined CBFs, we executed the safety filter defined in \eqref{safety_filter_single} with a simple P controller $\bm{k}(\bm{x}, \bm{\xi}) = k_p(\xi-x)$, where $k_p > 0$ is the proportional gain, and $\xi$ is generated using cubic spline interpolation. We set a goal state, $\bm{x}_{\textrm{goal}} = (0,\;0)$, marked by black stars in Fig. \ref{Fig3}, and used four different initial states: $\bm{x}_{\textrm{init},1} = (-9,\;15)$, $\bm{x}_{\textrm{init},2} = (-9,\;0)$, $\bm{x}_{\textrm{init},3} = (-7,\;-5)$, and $\bm{x}_{\textrm{init},4} = (-4,\;20)$.

The safety-critical control using the approximated CBFs results in distinct behaviors, as shown in Fig. \ref{Fig3}. The first method provides a relatively limited safe set ($\mathcal{P}_{A}$) that includes $\bm{x}_{\textrm{init},1}$, $\bm{x}_{\textrm{init},2}$, and $\bm{x}_{\textrm{init},3}$ but excludes $\bm{x}_{\textrm{init},4}$ and $\bm{x}_{\textrm{goal}}$, as shown in Fig. \ref{Fig3}(a). While the states starting from $\bm{x}_{\textrm{init},1}$, $\bm{x}_{\textrm{init},2}$, and $\bm{x}_{\textrm{init},3}$ remain within the defined safe set due to the designed safety filter, the system cannot reach a destination outside of the safe set. The safe set obtained by the second method ($\mathcal{P}_{B}$) is slightly larger than that of the first method, encompassing the goal state, as shown in Fig. \ref{Fig3}(b). The system approaches the goal due to the expanded safe set. Lastly, the combination of two CBFs based on the third method ($\mathcal{P}_{C}$) covers the entire region of $\mathcal{Z}_0$, as shown in Fig. \ref{Fig3}(c). In this case, the initial state $\bm{x}{\textrm{init},4}$ belongs to the safe set, allowing the system to reach the goal when starting from $\bm{x}_{\textrm{init},4}$. As explained in Remarks \ref{remark1} and \ref{remark2}, multiple CBFs can define a larger safe set when two CBFs are not identical. Due to the effect of the lower bound determined by $\alpha$ in the safety filter \eqref{safety_filter_single}, the terminal state cannot exactly reach the goal, which lies on the boundary of the safe sets.

\section{Conclusion}
This study addresses a critical issue stemming from the misalignment between state constraint functions and control barrier functions. We demonstrate that simple state constraints fail to guarantee the forward invariance of the constrained set, rendering the constraint functions incapable of serving as CBFs. To tackle this challenge, our study introduces three data-driven techniques for establishing CBFs while leveraging constraint functions: (1) Uniform scaling and Offset, (2) Non-Uniform scaling, and (3) Multiple CBFs. By employing a straightforward example involving a double integrator system, we illustrate the efficacy of each method in ensuring system safety during control.

In the near future, we intend to validate the applicability of our data-driven approaches with higher-order CBFs in more complicated and uncertain systems, such as legged robots. Also, addressing the complex task of identifying boundaries within $n$-dimensional state sets, where $n\geq4$, poses a formidable challenge. Consequently, we will explore efficient methods for boundary extraction from available datasets.

\bibliographystyle{IEEEtran}
\balance
\bibliography{l_css}

\end{document}